\documentclass[final,5p,times,twocolumn]{elsarticle}
\usepackage{amsmath,amssymb,amsfonts}
\usepackage{algorithmic}
\usepackage{graphicx}
\usepackage{textcomp}

\usepackage{xcolor}

\usepackage[hidelinks]{hyperref}

\usepackage{mathtools}
\usepackage{amssymb}
\usepackage{amsmath}
\usepackage{stmaryrd}
\usepackage{dsfont} 

\newcommand{\iverson}[1]{\mathds{1}\!\left(#1\right)} 

\newcommand{\minOp}{\wedge}
\newcommand{\maxOp}{\vee}
\NewDocumentCommand{\E}{m o}{
    \IfNoValueTF {#2}
        { \ensuremath{\mathbb{E}\left[#1\right] }}
        { \ensuremath{\mathbb{E}\left[#1 \middle| #2\right] }}
}
\NewDocumentCommand{\Prob}{m o}{
    \IfNoValueTF {#2}
        { \ensuremath{\mathbb{P}\left[#1\right] }}
        { \ensuremath{\mathbb{P}\left[#1 \middle| #2\right] }}
}

\newcommand{\T}{\top}

\newcommand{\ReLU}[1]{\left(#1\right)^{+}}

\newcommand\R{\mathbb{R}} 
\newcommand{\normdist}[2]{\mathcal{N}\left(#1,#2\right)} 
\newcommand{\bernoullidist}[1]{\operatorname{Ber}\left(#1\right)} 
\newcommand{\chisquareddist}{\chi^2}

\newtheorem{theorem}{Theorem}
\newtheorem{lemma}[theorem]{Lemma}
\newdefinition{remark}{Remark}
\newproof{proof}{Proof}
\newdefinition{definition}{Definition}

\newcommand{\stock}{X}
\newcommand{\stockCrit}{x_{c}}
\newcommand{\tEnd}{T}
\newcommand{\demand}{W}
\newcommand{\demandPredicted}{\bar{\demand}}
\newcommand{\demandMax}{\demand^{\max}}
\newcommand{\demandSupport}{[0,\demandMax]}
\newcommand{\buying}{U}

\newcommand{\shortageRate}{\alpha}

\newcommand{\shortageCum}{E}

\NewDocumentCommand{\cost}{o}{C\IfValueT{#1}{_{#1}}}
\NewDocumentCommand{\costHorizonPredSet}{m o}{\mathcal{\cost}^{#1}\IfValueT{#2}{(\infoVec_{#2})}}
\newcommand{\costHorizon}[1]{\cost^{\tHorizon}_{#1}}
\newcommand{\costHorizonPoint}{\bar{\cost}}
\newcommand{\costHorizonPredLower}[2]{\hat{\cost}^{#2}\!(\infoVec_{#1})}
\newcommand{\costHorizonPredUpper}[2]{\check{\cost}^{#2}\!(\infoVec_{#1})}

\newcommand{\costMax}{\widetilde{\cost}}
\newcommand{\costSupport}{[0,\costMax]}
\newcommand{\costEstimateMissRate}{\beta}

\newcommand{\tHorizon}{H}

\newcommand{\infoVec}{\mathcal{D}}
\newcommand{\sideInfo}{Z}

\newcommand{\policy}{\mu}
\newcommand{\control}[2]{\mu^{#1}\!\left(#2\right)}
\newcommand{\holdingCost}{h}

\newcommand{\demandBase}{\widehat{\demand}}

\newcommand{\pseudoquantile}{q}
\newcommand{\integrator}{g}

\newcommand{\errorBoundFun}{b}
\newcommand{\scoreBound}{g_{\circ}}
\NewDocumentCommand{\errorProcess}{m}{E_{#1}}

\newcommand{\tKnee}{\tEnd_{*}}
\newcommand{\rKnee}{\errorBoundFun_*}

\NewDocumentCommand{\tHistory}{}{B}

\begin{document}
\title{Certified Inventory Control of Critical Resources\tnoteref{t1}}

\author[1]{Ludvig Hult\corref{cor1}}\ead{ludvig.hult@it.uu.se}
\author[1]{Dave Zachariah}\ead{dave.zachariah@it.uu.se}
\author[1]{Petre Stoica}\ead{ps@it.uu.se}
\affiliation[1]{
    organization={Department of Information Technology, Uppsala University},
    addressline={Box 337},
    postcode={751 05},
    city={Uppsala},
    country={Sweden},
    }

\cortext[cor1]{Corresponding author}
\tnotetext[t1]{This work was supported in part by the Swedish Research Council under contract 2021-05022.}

\begin{abstract}
Inventory control is subject to service-level requirements, in which sufficient stock levels must be  maintained despite an unknown demand. We propose a data-driven order policy that certifies any prescribed service level under minimal assumptions on the unknown demand process. The policy achieves this using any online learning method along with integral action. We further propose an inference method that is valid in finite samples. The properties and theoretical guarantees of the method are illustrated using both synthetic and real-world data.
\end{abstract}


\maketitle

%
%
\section{Introduction}

\label{sec:introduction}
Inventory control using discrete-time models is a well-studied problem, where orders of items to hold in stock must anticipate future demand \cite{arrow1958, axsaeter2015}. By defining the costs of insufficient stocks, it is possible to find cost-minimizing policies using dynamic programming \cite{bertsekas_1976_dynamic,bertsekas_1995_dynamic,bensoussan_2011}. In practice, however, maintaining a certain service level of an inventory control system is a greater priority than cost minimization \cite{radasanu2016,bijulal2011}. Under certain restrictive assumptions on the demand process -- such as memoryless and identically distributed demand -- there are explicit formulations of the duality between service levels and costs \cite{vanhoutum2000}. Efforts to  relax such assumptions can be found in  \cite{larson2001, yan2024}. When the unknown demand distribution is learned from data, it is possible to provide probabilistic guarantees on the service level of an order policy in the special case of no stock being held between time periods \cite{huber_2019, levi_2015}.

In this letter, we formulate an inventory control relaxing most assumptions on the demand process and allowing for arbitrary time dependence of this process, while providing order policies with certifiable service-level guarantees. The relaxation is relevant in several critical inventory control problems where service levels are important, such as hospital inventory control \cite{saha2019,bijvank2012}. 

\begin{figure}[!t]
    \centerline{\includegraphics[width=\linewidth]{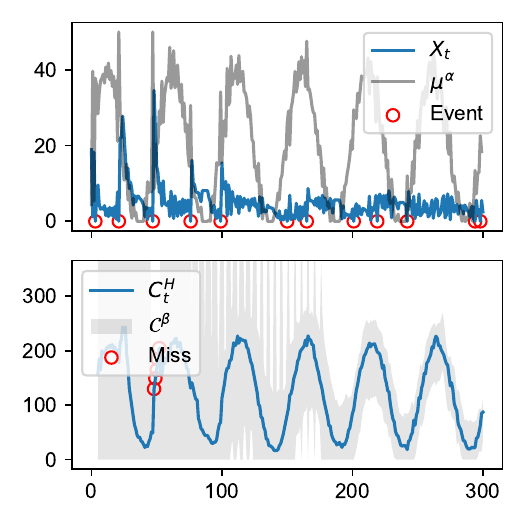}}
    \caption{Inventory control under an unknown (periodic) demand process. Top: Stock level and orders by policy $\policy^\shortageRate$ that is certified to have a service  level of at least $1-\shortageRate=95\%$. The empirical service level is larger than $96\%$. 
    Bottom: Policy operating costs $\costHorizon{t}$ over $\tHorizon=5$ steps ahead and a prediction interval $\costHorizonPredSet{\costEstimateMissRate}$ with a certified coverage level of $1-\costEstimateMissRate=95\%$. 
    For a full experimental description, see Section~\ref{sec:experiment noisey sine}.
    }
    \label{fig:1:main}
\end{figure}

We consider the management of a single item type with stock  level $\stock_t$ at time instant $t$ by determining an \emph{order} of $\buying_t$ additional units. Between $t$ and $t+1$, items in stock are consumed by an \emph{unknown} demand process $\demand_{t}$. 
A \emph{critical stock event} occurs if
\begin{equation}
   \stock_t \leq \stockCrit
\label{eq:defcriticalstockevent}
\end{equation}
for a fixed safety stock $\stockCrit$. (If $\stockCrit=0$, this is also known as a `stockout' event.) When \eqref{eq:defcriticalstockevent} is not satisfied, the normal stock level is insufficient to cover demand.
In this letter, we will describe an \emph{order policy} that guarantees a certain service level by bounding the number of critical stock events over a specified period. At any $t$, the policy has an operating cost, denoted $\costHorizon{t}$, over a future time horizon of length $\tHorizon$. To forecast its costs online, we develop a method for constructing prediction intervals $\costHorizonPredSet{\costEstimateMissRate}$ that will cover $\costHorizon{t}$ at a specified coverage level $1-\costEstimateMissRate$. Fig.~\ref{fig:1:main} provides an illustrative example of the  joint control and cost inference method proposed in this letter.



\textbf{Notation}
We employ the following notation:
$a \minOp b = \min\{a,b\}$, ${a \maxOp b} = \max\{a,b\}$, $\ReLU{x} = \max\{0,x\}$.

%
%
\section{Problem formulation}

The stock of an inventory system follows a discrete-time dynamical equation
\begin{equation}
    \stock_{t+1} = \ReLU{\stock_{t} + \buying_{t} - \demand_{t}},
\label{eq:dynamicalsystem}
\end{equation}
where order $\buying_t \geq 0$ is placed at time instant $t$ and a demand $\demand_t \in [0,\demandMax)$ is subsequently observed. The initial stock is $\stock_0 \geq 0$. For ease of exposition and without loss of generality, we will work with $\stockCrit=0$, so that critical stock events are indicated by $\iverson{\stock_{t} \leq 0}$. The generalization to $\stockCrit>0$ is straight forward.

At any $t$, we have access to historical stock data, demand data, and side information $\sideInfo_t$, collected in the variable $\infoVec_{t} = \{\stock_{0:t}, \demand_{0:t-1}, \sideInfo_{0:t}\}$. An order policy
$$\buying_t = \control{}{\infoVec_t}$$
is the control law of interest. We define its \emph{service level} as the fraction of non-critical stock events over a period $t=1, 2, \dots, \tEnd$. Our primary problem is to construct a policy $\control{\shortageRate}{\cdot}$ that ensures a service level of at least $1-\shortageRate$, that is 
\begin{equation}
    \boxed{
     \frac{1}{\tEnd}\sum_{t=1}^{\tEnd} \iverson{\stock_{t} > 0} \: \geq \: 1 - \shortageRate
     }
     \label{eq:admissible control}
\end{equation}
with a user-defined rate $\shortageRate\in[0,1]$. We call all policies that ensure \eqref{eq:admissible control} \emph{admissible} with a service level $1-\alpha$. 

Since any admissible policy $\policy^\shortageRate_0(\infoVec_t)$ can be replaced by $\policy^\shortageRate(\infoVec_t)= \policy^\shortageRate_0(\infoVec_t) \minOp \left(\demandMax-\stock_t\right)$ and remain admissible, we will assume all admissible policies fulfil $\control{\shortageRate}{\infoVec_t} \leq \demandMax - \stock_t$ without loss of generality. Note that the trivial policy $\control{\shortageRate}{\infoVec_t} \equiv \demandMax - \stock_t$ is also admissible but may hold a wastefully large stock beyond the specified service level. We are therefore seeking admissible policies with lower operating costs, that adapt to the information obtained about the demand process.

Let $\cost[t]$ denote the operating cost of policy at time $t$. Then its future operating cost over time horizon $\tHorizon$ is 
\begin{equation}
     \costHorizon{t} = \sum_{\tau=t}^{t+\tHorizon-1}  \cost[\tau].
     \label{eq:def:cost horizon}
\end{equation}
The choice of cost function is application dependent. One example is the cost of holding stock, e.g., $\cost_t \propto \stock_{t+1}$. The current period cost $\cost[t]$ is available in next period side information $\sideInfo_{t+1}$.
Operating costs depends critically on the demand process $\demand_t$, which is unknown. We therefore seek to construct prediction intervals $\costHorizonPredSet{\costEstimateMissRate}[t]$ that cover the future operating costs $\costHorizon{t}$ with a coverage level $1-\costEstimateMissRate$. That is,
\begin{equation}
    \boxed{\frac{1}{\tEnd-\tHorizon+1}
     \sum_{t=0}^{\tEnd-\tHorizon+1} \iverson{ \costHorizon{t} \in \costHorizonPredSet{\costEstimateMissRate}[t] } \geq 1-\costEstimateMissRate 
     }
     \label{eq:valid cost estimate}
\end{equation}
where $\costEstimateMissRate \in [0,1]$ is a user-defined rate.


%
%
\section{Method}

Let $\demandBase(\infoVec_t)$ be any predictor of the demand. Then the policy
\begin{equation}
    \control{}{\infoVec_t} = \ReLU{ \demandBase(\infoVec_t) - \stock_t }
\label{eq:datadrivenpolicy}
\end{equation}
is data-driven control law in which the orders are proportional to the predicted lack of stock. The demand predictor could assume any form, e.g.,
\begin{equation}
    \demandBase(\infoVec_t) = \phi^\T(\infoVec_t)\theta_t,
\label{eq:demandARXmodel}
\end{equation} 
where $\phi(\infoVec_t) = [1 \; \demand_{t-1} \: \cdots \: \demand_{t-d_{\demand}} \: \stock_{t} \: \cdots \: \stock_{t-d_\stock}]^\T$ is a simple autoregressive exogenous input (ARX) model. The more accurate models use, the less excess orders are made which keeps operating costs low. However, no matter how sophisticated the models are, the resulting policy \eqref{eq:datadrivenpolicy} is not ensured to be admissible under an unknown demand process $\demand_{t}$. 

We will introduce integral action which takes into account the cumulative critical stock events to render \eqref{eq:datadrivenpolicy} admissible with a service level $1-\shortageRate$. After deriving such an admissible order policy, we turn to the problem of inferring its future operation cost $\costHorizon{t}$. The control and inference problems will be tackled by introducing nonlinear gain functions that bound error processes.

Our proof technique is inspired by \cite{angelopoulos_2023}, which develops a one-step-ahead method for time-series prediction. However, the cited paper does not consider interventions and their feedback effects, as is the focus of automatic control, nor the problem of providing multi-step-ahead prediction guarantees for future operating costs of a dynamic interventional policy.

%
%
\subsection{Error Bound Functions}

Let the integer $\errorProcess{t}$ denote an \emph{error process} such that $\errorProcess{0}=0$ and $\errorProcess{t+1} \in \{\errorProcess{t},\errorProcess{t}+1\}$. For instance, it may count the number of critical stock events: $\errorProcess{t} = \sum_{t=1}^{\tEnd} \iverson{\stock_{t} \leq 0} $. We will now seek a function that can bound the error process.

\begin{definition}[Error bound function]
Any nondecreasing function $\errorBoundFun(t)$ that satisfies $0 \leq\errorBoundFun(t) \leq \shortageRate \tEnd$.
\end{definition}
A simple (piece-wise) linear example is
\begin{equation}
    \errorBoundFun(t) = 
    \begin{cases}
    0 & \text{ if } t \leq \tKnee \\
    \rKnee + (\shortageRate \tEnd - \rKnee)\frac{\ReLU{t-\tKnee}}{(\tEnd-\tKnee)} & \text{ else}
    \end{cases}
    \label{eq:defn:flexible error bound fun}
\end{equation}
with parameters $0\leq \rKnee \leq \shortageRate \tEnd $ and $0 \leq \tKnee < \tEnd$. The parameter $\tKnee$ specifies an initial period in which we tolerate no errors, which leads to conservative policies. This is useful if the demand predictor $\demandBase(\infoVec_t)$ requires some burn-in time. After a period $\tKnee$, there is a linear increase of the error bound function $\errorBoundFun(t)$ from $\rKnee$ to $\shortageRate\tEnd$. To determine a suitable integral action in a policy, we consider gain functions associated with $\errorBoundFun(t)$.

\begin{definition}[Associated gain]
An error function $\errorBoundFun(t)$ has an associated gain function $\integrator_t(\errorProcess{})$ if it satisfies
\begin{equation}
    \errorProcess{t}+1 \geq \errorBoundFun(t) \; \Rightarrow \; \integrator_t(\errorProcess{t}) \geq \scoreBound,
\label{eq:bound assumption:saturation}
\end{equation}
where $\scoreBound \in \R \cup \{\infty\}$ is a saturation level.
\end{definition}

\begin{lemma}\label{lemma:bounded}
    Consider any error bound function $\errorBoundFun(t)$ with an associated gain $\integrator_t(\errorProcess{})$. 
    If a saturated gain $\integrator_t(E_t) \geq \scoreBound$ implies no error growth, i.e., $\errorProcess{t+1} =\errorProcess{t}$, then the error process is bounded as
    \begin{equation} 
    \errorProcess{t}\leq \errorBoundFun(t) \leq \alpha \tEnd.
    \end{equation}
\end{lemma}
\begin{proof}
    The proof is by induction.
    Base case $\shortageCum_0 =0 \leq \errorBoundFun (0)$ holds by nonnegativity of the error bound function.
    Consider the inductive assumption $\shortageCum_{t-1}\leq \errorBoundFun (t-1)$. Two cases arise:
    
    Case i) $\shortageCum_{t-1} +1 \geq \errorBoundFun (t-1)$. 
    Then $\integrator_{t-1}\left( \errorProcess{t-1} \right) \geq \scoreBound$ and $\shortageCum_t = \shortageCum_{t-1} \leq \errorBoundFun (t-1) \leq \errorBoundFun (t)$. 
    
    Case ii) $\shortageCum_{t-1}+1 < \errorBoundFun (t-1)$. 
    Then $\shortageCum_t \leq \shortageCum_{t-1}+1 < \errorBoundFun(t-1)\leq  \errorBoundFun (t)$.

    Finally, $ \errorProcess{t} \leq \alpha \tEnd $ follows by the definition of $\errorProcess{t}$.
\end{proof}

Lemma~\ref{lemma:bounded} justifies the term `error bound function' for $\errorBoundFun(t)$. We will now turn to formulating a nonlinear gain function that adds integral action to the policy so as to render it admissible. The methodology will then be extended to infer the future operating costs.

%
%
\subsection{Admissible policy} \label{sec:admissible controller}

\begin{theorem}\label{thm:controller}
Define the nonlinear gain 
    \begin{equation}
        \integrator_t(\errorProcess{}) = \begin{cases}
            \tan \left( \frac{\pi}{2} \frac{\errorProcess{}+1}{\frac{t}{\tEnd}(\shortageRate\tEnd-2)+2}\right) & \text{ if }  \frac{\errorProcess{}+1}{\frac{t}{\tEnd}(\shortageRate\tEnd-2)+2} \in [0,1] \\
            \infty &  \text{ else} \\ 
            \end{cases}.
            \label{eq:specific integrator 1}
    \end{equation}
Then for any demand predictor $\demandBase(\cdot{})$, the order policy
    \begin{equation}
    \begin{split}
        \control{\shortageRate}{\infoVec_t} &= \ReLU{\demandBase(\infoVec_t) - \stock_{t} + \integrator_{t}\left(\sum_{\tau=1}^{t}\iverson{\stock_{\tau} \leq 0}\right) } \minOp \\ &\qquad  (\demandMax - \stock_{t})
        \end{split}
        \label{eq:proposed control}
    \end{equation}
     is  admissible with a service level $1-\shortageRate$ (see \eqref{eq:admissible control}).
\end{theorem}
\begin{proof}
By considering the error bound function $\errorBoundFun(t)$ in \eqref{eq:defn:flexible error bound fun} with $\tKnee=0$ and $\rKnee=2$, the gain function \eqref{eq:specific integrator 1} takes the form 
    \begin{equation}
        \integrator_t(\errorProcess{}) = \begin{cases}
            \tan \left( \frac{\pi}{2} \frac{\errorProcess{}+1}{\errorBoundFun(t)}\right) & \text{ if }  \errorProcess{}+1 \in [0,\errorBoundFun(t)] \\
            \infty &  \text{ else} \\ 
            \end{cases}.
    \end{equation}
This form shows that the gain is the associated gain for an error bound function  with saturation level $\scoreBound=\infty$ in \eqref{eq:bound assumption:saturation}.

The policy in \eqref{eq:proposed control} can be expressed as
     \begin{equation}\control{\shortageRate}{\infoVec_t} = \ReLU{ \demandPredicted_{t} - \stock_t }
     \label{eq:def:general policy}
     \end{equation}
where
$\demandPredicted_{t} = \inf \{ w \in \demandSupport : w - \demandBase(\infoVec_t) > \integrator_{t}\left( \errorProcess{t} \right)\}$
and we note that $\integrator_{t}\left( \errorProcess{t} \right) \geq \scoreBound  \Rightarrow \demandPredicted_{t}=\demandMax$. Using \eqref{eq:dynamicalsystem}, we can express the error process as:
    \begin{align}
        \errorProcess{t} &= \sum_{k=1}^{t} \iverson{\stock_{k} \leq 0}
        \label{eq:errorcum} \\
        &=\sum_{\tau=0}^{t-1} \iverson{ \demandPredicted_{\tau} \maxOp \stock_{\tau} \leq \demand_{\tau} }, \label{eq:cost specific error process}
    \end{align}
    showing that  $\demandPredicted_{t}=\demandMax \Rightarrow \errorProcess{t+1} = \errorProcess{t}$ (no error growth). Using Lemma~\ref{lemma:bounded}, it follows that $\errorProcess{\tEnd} \leq \shortageRate \tEnd$ and therefore the policy is admissible.
\end{proof}

%
%
\subsection{Valid cost inference}
We now turn to forecasting the operating costs \eqref{eq:def:cost horizon} of an admissible policy.

Similarly to the admissible policy construction that starts with a nominal predictor, the valid cost inference method starts with any  nominal prediction interval $\left[ \costHorizonPredLower{t}{\costEstimateMissRate}  , \costHorizonPredUpper{t}{\costEstimateMissRate} \right]$. To make things concrete, consider nominal intervals produced by a quantile estimator. Let $\costHorizonPoint(\infoVec_{\tau})$ be a simple point predictor of the cost and define its empirical error as $\epsilon_{\tau} = \costHorizon{\tau} - \costHorizonPoint(\infoVec_{\tau})$. Then a quantile estimator is given by:
\begin{equation}
    \costHorizonPoint_{a}(\infoVec_{t}) = \costHorizonPoint(\infoVec_{t}) + \inf\left\{ p : \frac{1}{t-\tHorizon+1}\sum_{\tau=0}^{t-\tHorizon} \iverson{\epsilon_\tau \leq p}   \geq a \right\}. \label{eq:sample average approx}
\end{equation}
and using it we can define a nominal prediction interval as
\begin{equation}
\left[ \costHorizonPredLower{t}{\costEstimateMissRate}  , \costHorizonPredUpper{t}{\costEstimateMissRate} \right] =
\left[ \costHorizonPoint_{\frac{\costEstimateMissRate}{2}}(\infoVec_{t})  , \costHorizonPoint_{1-\frac{\costEstimateMissRate}{2}}(\infoVec_{t}) \right].
    \label{eq:base interval model SAA}
\end{equation}
This nominal interval will be adjusted by introducing a gain as we show next.



\begin{theorem}\label{thm:costinference}
Assume non-negative bounded cost over a finite horizon, $\costHorizon{t} \in [0,\costMax] $, where the upper bound may be infinity.
Consider a horizon $\tHorizon \geq 2$ and any nominal prediction interval $[\costHorizonPredLower{t}{\costEstimateMissRate}, \costHorizonPredUpper{t}{\costEstimateMissRate}]$. Define an adjusted prediction interval
    \begin{equation}
        \costHorizonPredSet{\costEstimateMissRate}[t] =
        \left[ \ReLU{\costHorizonPredLower{t}{\costEstimateMissRate} -\pseudoquantile_{t} } , \left(\costHorizonPredUpper{t}{\costEstimateMissRate}+\pseudoquantile_{t} \right)\minOp \costMax \right],
        \label{eq:simple prediction interval}
    \end{equation}
where
    \begin{equation}\begin{split}
        \pseudoquantile_{t} = 
             \integrator_{t}\left( 
              \sum_{\tau=t-\tHorizon+1}^{t-1}  \iverson{\costHorizonPredSet{\costEstimateMissRate}[\tau] \neq \costSupport}  \right. \\
            + \left.\sum_{\tau=0}^{t-\tHorizon} \iverson{\costHorizon{\tau} \not \in \costHorizonPredSet{\costEstimateMissRate}[\tau]}
            \right)
    \end{split}\end{equation}
     and 
    \begin{equation}
        \integrator_t(\errorProcess{} ) = \begin{cases}
            \tan \left( \frac{\pi}{2} (2\frac{\errorProcess{} +1}{\errorBoundFun(t)}-1) \right) & \text{if } 0 < \errorBoundFun(t) \text{ and } E+1\leq \errorBoundFun(t) \\
            \infty & \text{else}\\ 
            \end{cases}
            \label{eq:specific integrator 2}
    \end{equation}
  is a gain associated with the error bound function $\errorBoundFun(t)$ in \eqref{eq:defn:flexible error bound fun} using $\costEstimateMissRate$ in lieu of $\shortageRate$. Then $ \costHorizonPredSet{\costEstimateMissRate}[t]$ has a coverage level $1-\costEstimateMissRate$ according to \eqref{eq:valid cost estimate}.
\end{theorem}
\begin{remark}
A similar result holds for the special case of $\tHorizon =1$, but is omitted here for sake of brevity.
\end{remark}
\begin{remark}
For $\errorBoundFun(t)$ we use $\rKnee=\tHorizon$ as a default value and the burn-in period $\tKnee$ is problem dependent.
\end{remark}

\begin{proof}
 Define
    \begin{equation}
    \begin{split}
           \costHorizonPredSet{\costEstimateMissRate}[t] &= \{ v \in \costSupport : \ReLU{v-\costHorizonPredUpper{t}{\costEstimateMissRate}} \\&+ \ReLU{ \costHorizonPredLower{t}{\costEstimateMissRate} - v} \leq \pseudoquantile_{t}\}
    \end{split}
    \end{equation}

    Let the number of observed miscoverage events at time $t$ be define as $\widetilde{E}_t = \sum_{\tau=0}^{t-\tHorizon} \iverson{ \costHorizon{\tau} \not\in \costHorizonPredSet{\costEstimateMissRate}[\tau] }$ that we want to bound by $\costEstimateMissRate \widetilde{\tEnd} $ where $\widetilde{\tEnd}  =  \tEnd - \tHorizon + 1$.
    Next, define the error process to be the number of known as well as potential miscoverage events:
    \begin{equation}
        \errorProcess{t} 
            =  \widetilde{E}_{t}
            + \sum_{\tau=t+\tHorizon+1}^{t-1}  \iverson{\costHorizonPredSet{\costEstimateMissRate}[\tau] \neq \costSupport}
    \end{equation}
    so that $\pseudoquantile_{t} = \integrator_t(\errorProcess{t})$. Let $\scoreBound = \costMax$ be the saturation level. Then a direct verification shows that $ \integrator_t(\errorProcess{t}) \geq \scoreBound \Rightarrow \errorProcess{t+1} = \errorProcess{t}$.
    We can therefore use Lemma~\ref{lemma:bounded} to show that $\errorProcess{\widetilde{\tEnd}} \leq \costEstimateMissRate \widetilde{\tEnd}$.
    Since $\widetilde{E}_{t} \leq \errorProcess{t}$, we have also proven that
    \begin{equation}
       \widetilde{E}_{\widetilde{\tEnd}} \leq  \errorProcess{\widetilde{\tEnd}}  \leq \errorBoundFun(\widetilde{\tEnd}) \leq \costEstimateMissRate \widetilde{\tEnd}
    \end{equation}
    and therefore that $\costHorizonPredSet{\costEstimateMissRate}[t] $ has a coverage level $1-\costEstimateMissRate$.
\end{proof}


%
%
\section{Numerical experiments}

The proposed order policy, and the online prediction interval for its future operational costs, are illustrated in a series of simulations using synthetic and real data. In all cases, we fix the prescribed service levels to $1-\shortageRate=95\%$ and  the coverage level to $1-\costEstimateMissRate=95\%$. 

The operating cost is taken to be cost of purchase plus holding; for a holding cost $h>0$, $\cost[t] = \buying_t + \holdingCost \stock_t$. A direct verification shows the finite horizon operating cost is upper bounded by $\costHorizon{t} = \tHorizon\demandMax(1+\holdingCost)$. For simplicity, we use $\holdingCost=1$.

The nominal demand predictor $\demandBase(\cdot)$ used is a linear-in-parameters autoregressive model \eqref{eq:demandARXmodel}, where parameters $\theta_t$ are tracked by recursive least squares (RLS) \cite[eq.~9.12]{soderstrom2001}. To reduce the burn-in time, the RLS parameters are initialized by tracking on historical data $\demand_{-\tHistory:0}$, $\buying_{-\tHistory:0}$, $\stock_{-\tHistory:0}$ recorded under an alternative policy, namely the empirical $(1-\shortageRate)$-quantile of demand; $\control{\shortageRate}{\infoVec_t} = \inf\left\{ p \, :\, \sum_{\tau = -\tHistory}^{t-1} \iverson{\demand_{\tau} \leq p}  \geq (1-\shortageRate)(\tHistory-t+1) \right\}$.

To construct the nominal prediction interval  \eqref{eq:base interval model SAA}, we use another linear-in-parameters autoregressive model with parameters tracked by RLS,
\begin{equation*}
\costHorizonPoint(\infoVec_{t}) =\varphi^\T(\infoVec_t) \vartheta_t.
\end{equation*}
This model predicts the future costs of the specific policy, and there are no data on this policy for $t\leq 0$. Therefore, no pre-training is done. The parameters are initialized with $\widehat{\vartheta}_0=[\costMax/2, 0  \dots{},0]$ and the first element $\varphi{}(\infoVec_t)$ is set to $1$, which yields a   conservative initial prediction.

\subsection{Synthetic data}
We explore the order policy and cost inference methods on three rather different demand processes.
All simulations on synthetic data use $\tEnd=300$, $\tHorizon=10$ and $\tHistory=150$, and the demand is upper bounded by $\demandMax=50$.

For demand prediction, we use model orders $d_\stock=d_\demand=2$ in \eqref{eq:demandARXmodel}, and set the RLS forgetting factor to $\lambda=0.99$.
For cost inference, we use the feature vector of a linear AR-5 model for $\tHorizon$ steps ahead prediction $\varphi(\infoVec_t) = \begin{bmatrix} 1 & \costHorizon{t-\tHorizon-4} & \costHorizon{t-\tHorizon-3} & \dots & \costHorizon{t-\tHorizon}\end{bmatrix}^\T$.

\subsubsection{Periodic demand}
The unknown demand process is assumed to follow
\label{sec:experiment noisey sine}
\begin{equation}
    W_t = 20  + 20\sin ( 2 \pi t/50)+ e_t
\end{equation}
where $e_t \sim \normdist{0}{1}$ and the demand is clipped to the interval $[0,50]$. This captures seasonality with random shocks.
In the cost inference, we use the RLS forgetting factor $\lambda=0.99$ and set a burn-in time of $\tKnee=40$ samples. 
Fig.~\ref{fig:1:main} shows the periodic orders resulting from the policy. We also see that the cost inference starts to become informative around $t=150$.

\subsubsection{Spiking demand} The unknown demand for items assumed to be proportional to the size of a latent infected population $I_t$,
\[W_t = 50I_t, \]
which models episodic demand that arises from infectious decreases.
Specifically, $I_t$ follows a stochastic `Susceptible-Infected-Removed' or SIR-type model:
\begin{equation}
\begin{split}
S_t &= S'_{t-1} - 0.5 S'_{t-1} I'_{t-1}\\
I_t &= I'_{t-1} + 0.5 S'_{t-1} I'_{t-1} - 0.2 I'_{t-1} \\
R_t &= (1 - e_t)R_{t-1} + 0.2 (I_{t-1} + 0.001 e_t)
\end{split}
\label{eq:SIRmodel}
\end{equation}
In \eqref{eq:SIRmodel}, we have that
\begin{align*}
S'_{t-1} &= S_{t-1} + (R_{t-1} - 0.001)e_t\\
I'_{t-1} &= I_{t-1} + 0.001 e_t 
\end{align*}
and $e_t \sim \bernoullidist{0.03}$. In each time step, with 3\% probability, the population loses immunity and 0.1\% of the population gets infected.

In the cost inference, we use the RLS forgetting factor $\lambda=0.995$ and set a burn-in time of $\tKnee=50$ samples. 
Fig.~\ref{fig:2:main} shows a spiking order pattern resulting from the policy. Note how the orders drop to zero during a longer period only to shoot up at the end.
\begin{figure}[!t]
    \centerline{\includegraphics[width=\linewidth]{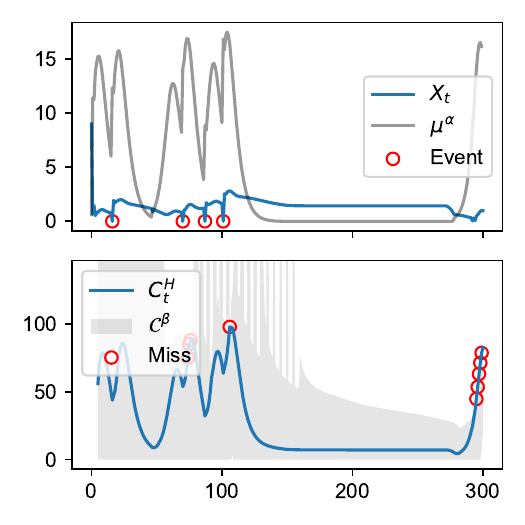}}
    \caption{Stock, purchase and cost estimate for the SIR model.
    The empirical service level is $98.7\%$ exceeding the prescribed $95\% = 1-\shortageRate$.
    The cost prediction interval covers the true cost in $97\%$ of cases, exceeding the prescribed $95\% = 1-\costEstimateMissRate$.
    }
    \label{fig:2:main}
\end{figure}

\subsubsection{Feedback demand} The unknown demand is assumed to depend on the stock level and therefore introduces feedback into system. Specifically, the demand process is
\begin{equation}
    \demand_{t+1} = (5 + \stock_{t} + e_t) \maxOp 49.999,
\end{equation}
where $e_t \sim \chisquareddist_1$.
In the cost inference, we use $\lambda=0.95$ and set the burn-in period to $\tKnee=30$ samples. 
Fig.~\ref{fig:3:main} shows a considerable tightening of the cost inference tightens considerably after $t=150$. 
\begin{figure}[t]
    \centerline{\includegraphics[width=\linewidth]{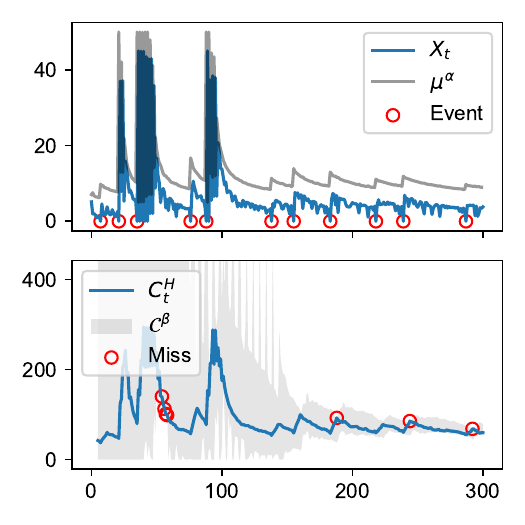}}
    \caption{Stock, purchase and cost estimate for the feedback demand model.
    The empirical service level is $96.3\%$ exceeding the prescribed $95\% = 1-\shortageRate$.
    The cost prediction interval covers the true cost in $97.7\%$ of cases, exceeding the prescribed $95\% = 1-\costEstimateMissRate$.
    }
    \label{fig:3:main}
\end{figure}

\subsection{Electricity dataset}

The previous results validate the guarantees of the order policy and cost inference methodology for rather different synthetic demand processes. As an example of a real-world demand process, we now use the electricity demand variable $\demand_t$ (NSWDemand) in the Elec2 dataset\cite{harries1999}. This process represents seasonality and shocks as well as other distribution shifts. It is therefore both interesting and challenging to analyze.
The full dataset records the demand every 30 minutes (48 samples per day) for ca 900 days,
normalized so that $\demandMax =1$.
We set the cost horizon to $\tHorizon=1\times{}48$ corresponding to 1 day. The pretraining window length to 3 days ($\tHistory=3\times{}48$) and algorithm is run over 12 weeks of time periods ($\tEnd = 12\times7\times48$).
The first $\tHistory+\tEnd$ datapoints are used for selecting suitable hyperparameters. The algorithm is then run on the subsequent $\tHistory+\tEnd$ samples.

Concerning hyperparameter selection, we chose demand model orders $d_\stock=0$, $d_\demand=48$ in \eqref{eq:demandARXmodel}, and set the RLS with forgetting factor to $\lambda=0.99$.
The cost inference feature map $\varphi$ was set to include an autoregressive model of order 24 with additional Fourier coefficients representing periods of 3, 6, 12, 24 hours, and 7 days. The burn-in time was $\tKnee=10\times 48$ and the forgetting factor was $\lambda=0.995$.

Fig.~\ref{fig:elec2:main} shows strong periodic variations as well as sudden spikes in the orders. Despite significant variability in demand, the order policy ensures the prescribed service level $1-\shortageRate$. We also note that the cost inference tightens as in the previous case, but around $t=3000$  demand fluctuations lead to several miscoverage events. These inflate the prediction intervals, which subsequently  shrink again.
\begin{figure}[!t]
    \centerline{\includegraphics[width=\linewidth]{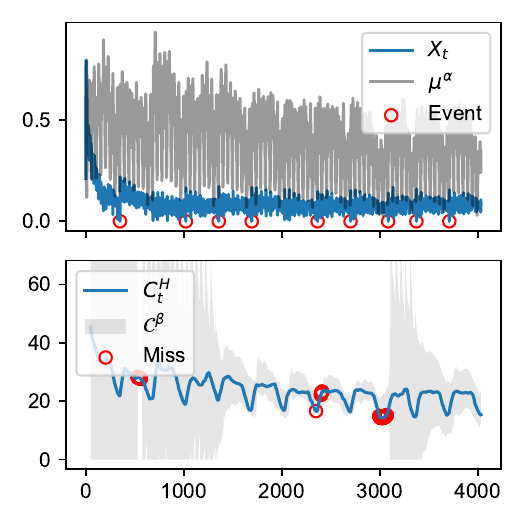}}
    \caption{Stock, purchase and cost estimates using the demand recorded in the Elec2 dataset (and is ormalized to [0,1].)
    The empirical service level is $99.8\%$ exceeding the prescribed $95\% = 1-\shortageRate$.
    The cost prediction interval covers the true cost in $97\%$ of cases, exceeding the prescribed $95\% = 1-\costEstimateMissRate$.
    }
    \label{fig:elec2:main}
\end{figure}

We use this example to illustrate in Fig.~\ref{fig:elec2:errors} the error processes $\errorProcess{t}$ associated with the order policy and cost inference, respectively. These processes are bounded by $\errorBoundFun(t)$, which determines the associated gains. 
\begin{figure}[!t]
    \centerline{\includegraphics[width=\linewidth]{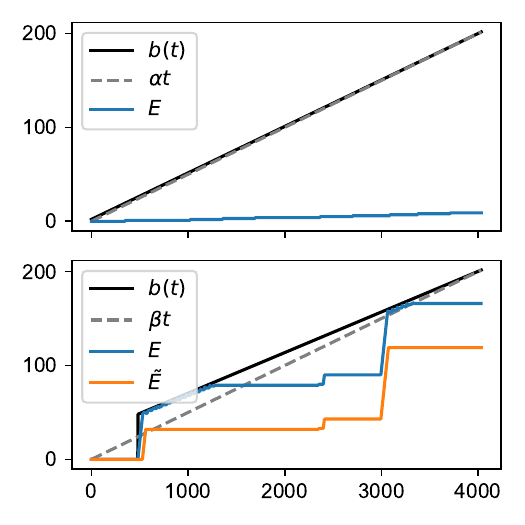}}
    \caption{Error processes $\errorProcess{t}$ for experiments on the Elec2 dataset, which are provably bounded by $\errorBoundFun(t)$ in \eqref{eq:defn:flexible error bound fun} with parameters $\tKnee$, $\rKnee$ and $\shortageRate$.
    Top: $\errorProcess{t}$ is the number of critical stock events. The default $\rKnee=2$  $\tKnee=0$ are used.
    Bottom: $\errorProcess{t}$ is the number of observed ($\widetilde{E}_t$) plus potential miscoverage events in the cost inference. The default $\rKnee=\tHorizon$ and $\shortageRate=\costEstimateMissRate$.}
    \label{fig:elec2:errors}
\end{figure}

\section{Discussion}

We have considered an inventory control problem with an unknown demand process in which order policies can be constructed with certified service levels and their operating cost can be forecast with a prescribed coverage level. These guarantees hold under very weak assumptions on the unknown demand process. We have verified the methods numerically and illustrated them on synthetic and real data.

The method builds on any nominal predictors of the demand and operating cost, which are adjusted using a nonlinear integration of errors made in the policy or cost inference, respectively. Several design choices are in fact possible here.  In the simulation studies we used classical linear autoregressive models, but the use of more sophisticated predictors, such neural networks, is rather straightforward.

The error bound function and the associated gain function of choice affect the conservativeness of the method. In our experiments, we used the error bound function \eqref{eq:defn:flexible error bound fun} with a simple parametrization, where we adjust the burn in time $\tKnee$ reflecting the fact that different models need more data to predict well. More complicated error bound functions are conceivable, which depend not only on time but also on the accuracy of the nominal predictor models.


Further exploratory work may study the impact of the choice of prediction models, as well as error bound functions and associated gains, on the operating costs and the tightness in cost inference. This may lead to finding more efficient admissible policies and cost inferences.

\section{CRediT authorship contribution statement}
\textbf{Ludvig Hult:} Conceptualization, Formal Analysis, Investigation, Software,

\textbf{Dave Zachariah:} Funding acquisition, Writing – review \& editing, Conceptualization, Supervision

\textbf{Petre Stoica:} Writing – review \& editing

\bibliographystyle{IEEEtran}
\bibliography{main.bib}

\end{document}